\documentclass{article}

     \usepackage[preprint, nonatbib]{neurips_2020}

\usepackage{amsmath,amsthm,amsfonts,amssymb,amscd}
\usepackage{lastpage}
\usepackage{braket}
\usepackage{enumerate}
\usepackage{fancyhdr}
\usepackage{mathrsfs}
\usepackage{xcolor}
\usepackage{graphicx}
\usepackage{listings}
\usepackage{hyperref}
\usepackage{float}
\usepackage{bbm}
\usepackage{subfig}
\usepackage{algorithm}
\usepackage{algpseudocode}
\usepackage[toc,page]{appendix}
\usepackage[utf8]{inputenc} 
\usepackage[T1]{fontenc}    
\usepackage{url}            
\usepackage{booktabs}       
\usepackage{nicefrac}       
\usepackage{microtype}      
\usepackage{booktabs}

\usepackage[english]{babel}
 
\newtheorem{theorem}{Theorem}

\newtheorem{lemma}{Lemma}[theorem]

\newtheorem{proposition}{Proposition}[theorem]
\newtheorem{remark}{Remark}

\newcommand{\m}[1]{\mathbf{#1}}
\newcommand{\norm}[1]{\left\lVert#1\right\rVert}
\newcommand{\N}{\mathcal{N}}
\newcommand{\E}{\mathbb{E}}
\newcommand{\pr}{\mathbb{P}}
\newcommand{\G}{\mathbf{G}}
\newcommand{\W}{\mathbf{W}}
\newcommand{\B}{\mathbf{B}}
\newcommand{\A}{\mathbf{A}}
\newcommand{\C}{\mathbf{C}}
\newcommand{\D}{\mathbf{D}}
\newcommand{\J}{\mathbf{J}}
\newcommand{\x}{\mathbf{x}}
\newcommand{\z}{\mathbf{z}}
\newcommand{\hatx}[2]{\mathbf{\hat{x}}^{#1}_{#2}}
\newcommand{\M}{\mathbf{M}}
\newcommand{\etab}{\boldsymbol\eta}
\newcommand{\floor}[1]{\left\lfloor #1 \right\rfloor}

\newcommand{\one}{\mathbbm{1}}

\newcommand{\R}{\mathbb{R}}

\definecolor{airforceblue}{rgb}{0.36, 0.54, 0.66}

\hypersetup{%
  colorlinks=true,
  linkcolor=blue,
  citecolor=airforceblue,
  linkbordercolor={0 0 1}
}

\title{Compressing Heavy-Tailed Weight Matrices for Non-Vacuous Generalization Bounds}

\author{%
  John Y. ~Shin \\
  RiskEcon® Lab \thanks{Please see the acknowledgements section for affiliation details.}\\
  Courant Institute of Mathematical Sciences\\
  New York University\\
  New York, NY 10003 \\
  \texttt{jys308@nyu.edu} \\
}

\begin{document}

\maketitle

\begin{abstract}
Heavy-tailed distributions have been studied in statistics, random matrix theory, physics, and econometrics as models of correlated systems, among other domains. Further, heavy-tail distributed eigenvalues of the covariance matrix of the weight matrices in neural networks have been shown to empirically correlate with test set accuracy in several works (e.g. \cite{ mahoney2019traditional}), but a formal relationship between heavy-tail distributed parameters and generalization bounds was yet to be demonstrated. In this work, the compression framework of  \cite{arora2018stronger} is utilized to show that matrices with heavy-tail distributed matrix elements can be compressed, resulting in networks with sparse weight matrices. Since the parameter count has been reduced to a sum of the non-zero elements of sparse matrices, the compression framework allows us to bound the \emph{generalization gap} of the resulting compressed network with a non-vacuous generalization bound. Further, the action of these matrices on a vector is discussed, and how they may relate to compression and \emph{resilient classification} is analyzed. 
\end{abstract}

\section{Introduction}

The \emph{generalization gap} describes the difference in performance of a statistical model over the underlying distribution and over the training dataset \cite{jiang2018predicting}. Bounding the \emph{generalization gap} with quantities that elucidate mechanisms for generalization are of interest both from a theoretical and practical point-of-view, where articulating the mechanisms may guide architecture selection as well as contributing to new ways of measuring the performance of neural networks. In the traditional machine learning literature, methods to bound the \emph{generalization gap} include methods utilizing Rademacher complexity, Vapnik–Chervonenkis (VC) dimension, and uniform stability. Experimental work employing progressively more randomized labels identified limitations for such methods when applied to neural networks  (\cite{zhang2016understanding}), and the aforementioned bounds are generally found to be \emph{loose} for neural networks. 

Several lines of research, both theoretical and empirical, have been undertaken to better understand the \emph{generalization gap} of neural networks. These include PAC-Bayesian methods (\cite{mcallester1999some, mcallester1999pac, neyshabur2017pac}), the flatness of local minima (\cite{keskar2016large, foret2020sharpness}), compression (\cite{arora2018stronger}), margin and margin distributions (\cite{bartlett2017spectrally, jiang2018predicting}), as well as power-law behavior (\cite{mahoney2019traditional, martin2020heavy}). In companion, recent experimental works (\cite{jiang2019fantastic, dziugaite2020search}) have comparatively assessed many of these generalization measures and their performance on a variety of datasets. While the metric of \emph{sharpness} appears to serve as a complexity measure for generalization, the latter of the two works raises questions on this connection. 

Heavy-tailed distributions are distributions whose moment-generating function diverges for all positive parameterizations (\cite{foss2011introduction}), and have been used as models of correlated systems (\cite{bakhshizadeh2020sharp, aggarwal2018goe, levy, topeigenvalue, tarquini2016level, financial}). In this work, we assume that the fully trained weight matrices of a \emph{fully connected neural network} (FCNN) have power-law structure (\cite{martin2020heavy} and  \cite{ mahoney2019traditional}), and  using the compression framework of \cite{arora2018stronger}, prove bounds on generalization given this structure. The bound depends on a compressed parameter count that is the sum of the non-zero entries of sparse matrices. This results in a non-vacuous generalization bound, i.e., a bound that is less than one. The key idea of how to compress these matrices was inspired by the work of \cite{aggarwal2018goe} and \cite{bakhshizadeh2020sharp}. 
The work will be structured as follows:

\begin{enumerate}
    \item Review the compression framework of \cite{arora2018stronger} and informally state the main theorem of their work. Describe the problem set-up. 
    \item Demonstrate that heavy-tailed matrices can be compressed and the resulting sparsity of the compressed weight matrix using techniques from standard measure concentration.
    \item Prove a bound on the \emph{generalization gap} of an FCNN with compressed weight matrices using the compression framework.
    \item Analyze the action of heavy-tailed weight matrices on a vector, and show how they relate to compression and \emph{resilient classification}, as well as to the stable rank of a matrix. 
    \item Provide an empirical evaluation of the matrix compression, and demonstrate that the compression results in a comparable accuracy. 
\end{enumerate}

\section{Background and Problem Setup}

This work builds on the framework of  \cite{arora2018stronger}, and the basic notions of the work are reviewed. The compression framework allows one to bound the \emph{generalization gap} of a compressed version of a given network. In the original work, the weight matrices of an FCNN are compressed using a form of \emph{Johnson-Lindenstrauss} compression to compress the weights themselves \cite{johnson1984extensions}.

We work within the multiclass classification setup. Suppose we have a sample $(\x, y) \sim \mathcal{D}$, where $\x \in \R^b$ with a label $y \in \{1, \cdots, k\}$. A multiclass classifier $f$ maps input $\x$ to $f(\x) \in \R^k$. With $\gamma > 0$ as the margin, the expected margin loss is given as:

\begin{equation}
L_{\gamma}(f) = \pr_{(\x,y) \sim \mathcal{D}}\left[f(\x)[y] \leq \gamma + \max_{j \neq y} f(\x) [j] \right]
\end{equation}

For $\gamma = 0$, this represents the classification loss. In addition, we will denote with $\hat{L}_{\gamma}$ the empirical estimate of the margin loss over the training dataset $S$. 

A goal of the compression framework is to reduce the number of parameters of the original network through compression, in the hopes that the new parameter count of the compressed network will yield a realistic bound on the classification loss, while minimizing the error introduced through compression.  

\begin{theorem} (\cite{arora2018stronger} - Informally Stated.)
Suppose $g$ is a family of classifiers, which are compressed versions of the classifier $f$. Let $q$ be the parameter count of $g$ each of which can have at most $r$ discrete values. Let $S$ be a training set with $m$ samples. If the output of the  classifier $f$ is within $\gamma$ of $g$ over the training set, then with high probability:
\begin{equation}
L_0(g) \leq \hat{L}_{\gamma}(f) + O\left(\sqrt{\frac{q \log r}{m}}\right)
\end{equation}

\end{theorem}

\section{Compressing Heavy-Tails and Sparsity}

Heavy-tailed distributions are probability distributions whose moment-generating function diverges for all positive parameterizations. A class of heavy-tailed distributions are the L\'evy $\alpha$-stable distributions. The L\'evy $\alpha$-stable distributions are a family of distributions parameterized by a stability parameter, $\alpha \in (0, 2]$ for which $\alpha = 2$ corresponds to the Gaussian distribution and $\alpha = 1$ corresponds to the Cauchy distribution. \footnote{The Cauchy distribution is defined with the probability density function of $p(x) = \frac{1}{\pi \gamma \left[1+ \frac{(x-x_0)}{\gamma}^2 \right]}$, where $\gamma$ is a scale paramater, and $x_0$ is a location parameter.} 
For $\alpha \neq 2$, this family is heavy-tailed. For $\alpha < 2$, the variance of these distributions are undefined, and for $\alpha \leq 1$, the mean is undefined. Random variables drawn from these distributions are stable under addition, that is, given two independent copies $(X_1, X_2)$ of a random variable $X$, their sum $a X_1 + b X_2$ for any $a, b > 0$ is given by $cX+d$ for some $c > 0$ and $d$.  The $\alpha$-stable property generalizes the classical \emph{Central Limit Theorem} to random variables that do not exhibit finite variance. Except for the special cases, this family of distributions cannot be written in terms of elementary functions \cite{nolan2009univariate}.\footnote{If a random variable $X$ is drawn from an $\alpha$-stable distribution, the  characteristic function of $X$ is given by: $$\E\left[ \exp(i u X) \right] = \begin{cases}
\exp\left( - c^{\alpha}|u|^{\alpha} \left[1 + i \beta (\tan\frac{\pi \alpha}{2})(\text{sign } u)(|c u|^{1 - \alpha} - 1) \right] + i \delta u \right), \quad \alpha \neq 1\\
\exp\left( - c |u| \left[1 + i\beta \frac{2}{\pi} (\text{sign } u) \log(\gamma |u|) \right] + i \delta u \right), \quad \alpha = 1
\end{cases}$$ where $c \geq 0$ is a scale parameter, $\beta \in [-1, 1]$ is an asymmetry parameter, $\delta \in \R$ is a location parameter, and $\alpha \in (0, 2]$ is the stability index.} In Figure \ref{levystable}, we plot the L\'evy $\alpha$-stable distributions as a function of the stability parameter $\alpha$. As shown in the figure, modifying the stability parameter $\alpha$ affects the tail of the distribution.

\begin{figure}
    \centering
    \includegraphics[scale=0.55]{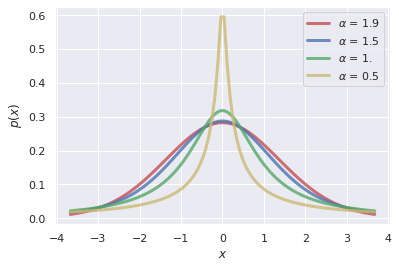}
    \caption{The L\'evy $\alpha$-stable distributions as a function of the stability parameter $\alpha$. The stability parameter $\alpha$ controls the tail of the distribution.}
    \label{levystable}
\end{figure}

\subsection{$\alpha$-stable Distributions and Stochastic Gradient Descent}

While we assume that the fully trained weight matrices of an FCNN have power-law structure in the following sections, we do not demonstrate how this structure arises during the course of training. Several works have explored the role of heavy-tailed behavior in the weights of a neural network via stochastic gradient descent (SGD)  (e.g., \cite{hodgkinson2020multiplicative, gurbuzbalaban2020heavy} and \cite{csimcsekli2019heavy}). We give a high-level description of these works. 

Suppose we have a neural network classifier $f(\x, \m{w})$ which maps an input $\x$ to $f(\x, \m{w}) \in \R^k$ given the parameters $\m{w}$. Suppose $L(f, y)$ is the loss function with $y$ as the true label. In deep learning, it is common to use a version of SGD for training (See for example \cite{ruder2016overview}), that is, a single sample or batch of samples are used to approximate the full gradient:

\begin{equation}
\m{w}_{t+1} = \m{w}_{t} - \eta \nabla_{\m{w}_t}\tilde{L}(f(\x_i, \m{w}_t), y_i)
\end{equation}

Where $\eta$ is the step-size, and $\nabla_{\m{w}} \tilde{L} \equiv (1/b) \sum_{i \in B} \nabla_{\m{w}} L(f(\x_i, \m{w}) , y_i)$, or the approximation of the gradient over a random subset, $B$, of the training set with cardinality $b$. 

Over the iterations $t$, this can be seen as a sum of random variables. It is hypothesized that this sum will be attracted to an $\alpha$-stable law over the course of training. In the work of \cite{gurbuzbalaban2020heavy}, it was shown that heavy tails in the weights can arise in SGD even in simple settings where the input data is Gaussian, and not heavy-tailed. It is  noted that the developed proof of the generalization bound in this work only assumes power-law structure (not $\alpha$-stability), and a portion of the analysis of the matrix action utilizes $\alpha$-stability. The overall logic can be stated as follows:

\begin{itemize}
    \item There is theoretical and empirical support that weights trained using SGD fall into an $\alpha$-stable distribution. $\alpha$-stable distributions behave asymptotically like a power-law. 
    \item It is assumed that the final weights of an FCNN are distributed as a power-law, and derive a generalization bound given this assumption. The $\alpha$-stable property is not needed to derive this bound. In principle, a similar bound can be derived from any heavy-tailed distribution. 
\end{itemize}

\subsection{Compressing Heavy-Tails}

Next, we wish to show that matrices with matrix elements distributed as a power-law can be compressed. Let $\W$ be a matrix of dimension $N_1 \times N_2 $ , where each matrix element is drawn $\W_{i,j} \sim (\alpha w_m^{\alpha})w^{-(\alpha+1)}, i \in \{1, \cdots, N_1\}, j \in \{1, \cdots N_2\}$ from a power-law distribution with exponent $\alpha$ and with support $w \in [w_m, \infty)$, where $w_m > 0$. The term $\alpha w_m^{\alpha}$ is a normalization factor to ensure the distribution integrates to 1. Suppose for some threshold $\tau > x_{m}$, $\A = \W_{i,j}\one[\W_{i,j} \leq \tau]$ and $\B = \W_{i,j}\one[\W_{i,j} > \tau]$. Thus, $\W = \A + \B$. We wish to replace the matrix $\A$ with a matrix $\sqrt{t}\m{G}$, a matrix with i.i.d normal distribution entries scaled by $\sqrt{t}$. We define the matrix $\W(t) \equiv \sqrt{t}\G + \B$. We wish to compute the error introduced by replacing $\W$ with $\W(t)$. Let $\overline{\W} \equiv \W - \W(t) = \A - \sqrt{t}\G$. We state the proposition:

\begin{proposition}
For any $\m{u} \in \R^{N_1}, \m{v} \in \R^{N_2}$, $\frac{2}{\tau^2 + t} = 2\frac{\log(\frac{3}{\eta})}{\epsilon^2}$: 
\begin{equation}
\pr \left\{ \left| \m{u}^T \overline{\W} \m{v} \right| > \epsilon \norm{\m{u}}\norm{\m{v}} \right\} \leq  \eta 
\end{equation}
\label{sherpa3}
\end{proposition}

 For a given error $\epsilon$ and failure probability $\eta$, we can bound the error of replacing the heavy-tailed weight matrix with a Gaussian approximation. The proof of this is given in Appendix \ref{appendixA}, and utilizes basic measure concentration techniques. 

\subsection{$k$-Sparse Weight Matrices}

Next, we demonstrate the sparsity of the $\B$ matrices. We know $\B_{k,l} = \W_{k,l}\one[\W_{k,l} > \tau]$, $k \in \{ 1, \cdots, N_1\}, l \in \{1, \cdots, N_2\}$, where $\W_{k,l} \sim (\alpha w_m^{\alpha})w^{-(\alpha+1)}$ is drawn from a power law distribution with exponent $\alpha$ and with support $w \in [w_m, \infty)$, where $w_m > 0$. The probability that any given matrix element of $\B$ is non-zero is given by:

\begin{gather}
\int_{\tau}^{\infty}p(w) dw = \int_{\tau}^{\infty} (\alpha w_m^{\alpha})w^{-(\alpha+1)} dw = w_m^{\alpha}\tau^{-\alpha}
\end{gather}

With $\alpha > 0, \tau > w_m > 0$. So with probability $ w_m^{\alpha}\tau^{-\alpha}$, a given matrix element is non-zero, and with probability $1 -  w_m^{\alpha}\tau^{-\alpha}$, the matrix element is zero. Whether a given matrix element is zero or non-zero corresponds to a Bernoulli random variable with probability $p =  w_m^{\alpha}\tau^{-\alpha}$. Thus, the number of non-zero elements can be bounded by bounding the binomial distribution:

\begin{proposition}
Let $k$ be a chosen sparsity threshold, $p =w_m^{\alpha}\tau^{-\alpha}$ be the probability of a non-zero matrix element, $n = N_1 \times N_2$, and $X$ the number of non-zero matrix elements. For $k > np$, the probability that $X \geq k$ is bounded by:

\begin{equation}
\pr(X \geq k)  \leq \left(\frac{nw_m^{\alpha}\tau^{-\alpha}}{k}\right)^k \left(\frac{1-w_m^{\alpha}\tau^{-\alpha}}{1-\frac kn}\right)^{n-k}
\end{equation}
\label{binomial}
\end{proposition}

Thus the probability that $X$ is not $k$-sparse is a function of the underlying power-law, and sufficiently small for choice of $k$ and $\tau$. The full proof of this is given in Appendix \ref{appendixB}.

\section{A Bound on FCNNs}

Finally, we wish to prove a bound on the generalization error of an FCNN whose weights are compressed given the method outlined above. Suppose we have an FCNN $f$ with layer weight matrices $\W^i \in \R^{h_i \times h_{i-1}}, i \in \{1, \cdots, d\}$. Let $\W(t_i)^i$ be the approximation to this matrix as outlined above. We begin by stating the matrix approximation algorithm in Algorithm \ref{alg:algorithm1}.

\begin{algorithm}[ht]
\caption{Matrix Approximation ($\W, \epsilon, \eta, t, \tau$)}
\label{alg:algorithm1}

\begin{algorithmic}[1]
\Require Layer matrix $\W^i \in \R^{h_i \times h_{i-1}}$, error parameters ($\epsilon_i, \eta_i$), cutoffs ($\tau_i, t_i$), $\lambda_i = \frac{2}{\tau_i^2+t_i}$
\Ensure Returns $\W(t_i)^i$ s.t. $\forall $ fixed vectors $\m{u}, \m{v}$:
\begin{equation}
\pr \left\{|\m{u}^T \W(t_i)^i\m{v} - \m{u}^T \W^i \m{v}| \geq \epsilon_i \norm{\m{u}} \norm{\m{v}}\right\} \leq \eta_i 
\end{equation}
\State{Set $\lambda_i = 2\frac{\log(\frac{3}{\eta_i})}{\epsilon_i^2}$}
\For{$i = 1, \cdots, d$}
\State{Replace $\W^i$ with $\W(t_i)^i$}
\EndFor
\end{algorithmic}
\label{Algo1}
\end{algorithm}

The algorithm replaces all the weight matrices of an FCNN with a Gaussian approximation given an error tolerance $\epsilon$, probability of failure $\eta$, threshold $\tau$ and variance $t$. Using this approximation, we then prove a theorem on the generalization of a compressed FCNN, compressed with this algorithm. 

\begin{theorem}
Let $f$ be an FCNN whose weight matrices are given by $\W^i \in \R^{h_i \times h_{i-1}}$, $i = \{ 1, \cdots, d\}$. Let $g$ be the compressed version of $f$, compressed with Algorithm \ref{Algo1}, with margin $\gamma$ over the training set $S$ of cardinality $m$. Let $k_i$ be the sparsity of the compressed weight matrices, $\W(t_i)^i$. Then, with high probability, the classification loss of $g$ (ignoring logarithmic factors) is bounded by:

\begin{equation}
L_0 (g) \leq \hat{L}_{\gamma}(f) + \tilde{O}\left(\sqrt{\frac{\sum_{i=1}^d k_i}{m}} \right)
\end{equation}
\label{generalization}
\end{theorem}

The proof strategy is the same strategy outlined in \cite{arora2018stronger}, which bounds the error of the approximation using an induction argument. The full proof of this theorem is given in Appendix \ref{appendixC}. 

\begin{remark}
\normalfont It should be evident that the theorem favors matrices that have a higher value of $\alpha$ (i.e. matrices with fewer non-zero matrix elements). We note that the empirical results in the work of \cite{martin2020heavy} measures the heavy-tailedness of the eigenvalue distribution of the correlation matrix of the weights, $\W\W^T$, rather than the matrix elements of $\W$ itself as we have done here. Their work finds empirical correlation of heavier tails in the eigenvalue distribution with test set accuracy. The relationship between the underlying probability distribution of the matrix elements of $\W$ and the eigenvalues of $\W\W^T$ is an area requiring further theoretical exploration. Another work which studied the matrix elements of $\W$ itself (\cite{fortuin2021bayesian}) indicates that the matrix elements are more heavy-tailed than a Gaussian, but is well fit by a Laplace distribution, which is consistent with the bound here. They also find that later layers become progressively more heavy-tailed. We discuss this in the following section. 
\end{remark}

\begin{remark}
\normalfont It is also noted that the proof of Theorem \ref{generalization} did not require much deviation  of the proof techniques contained in \cite{arora2018stronger}. The four constants utilized in the proof of Theorem \ref{generalization} may be artifacts of the proof technique. A potential area of future work would be to develop new proof techniques to remove some or all of these constants. In the original work of \cite{arora2018stronger}, it was found that of these four constants, the layer cushion was the best indicator of generalization. The inverse of the layer cushion $\frac{1}{\mu_i}$ (Appendix \ref{appendixC}) is lower-bounded by the square root of the stable-rank of $\W$. It was observed in \cite{arora2018stronger} that a larger layer cushion improved generalization (i.e. smaller inverse layer cushion). The relationship between the stable-rank and power-laws will be discussed in the next section. 
\end{remark}

\begin{remark}
\normalfont In our proof we have also assumed that the matrix elements are drawn i.i.d. For the fully trained weight matrices of a neural-network, this assumption is likely to be unrealistic. Heuristically, this approximation might be valid if we view the high components as the correlated ones (which we keep through the approximation), and the low components as approximately i.i.d. Relaxing the i.i.d. assumption and examining models of correlation between the weights is an important potential area of future work. 

\end{remark}

\section{Discussion and Analysis}

Given the three remarks above, we wish to analyze the action of the matrix $\W$ on $\x$, i.e. $\W \x$. Given the empirical evidence that early layers are light-tailed and the later layers become heavy-tailed in the work of \cite{fortuin2021bayesian}, we give a heuristic analysis on how 1) early layers with lighter tails (large $\alpha$) may be acting as compression layers, and that 2) later layers with heavier tails (small $\alpha$) may be indicative of \emph{resilient classification}, which we will define in it's eponymous section. We will also examine 3) how heavy-tailed matrices relate to the stable rank of a matrix. In this section, we will assume a cross-entropy loss with a softmax function after the final layer.

\subsection{Spiked Johnson-Lindenstrauss Compression}

In this section, we wish to show how the earlier layers of an FCNN may be acting as compression layers. As indicated above, the original weight matrix $\W$ with heavy-tailed matrix elements can be approximated by the matrix $\W(t) \equiv \sqrt{t} \G + \B$. We can view this as an approximate \emph{Johnson-Lindenstrauss} compression (\cite{johnson1984extensions}) where $\sqrt{t}\G$ acts as a compression matrix while $\B$ can be seen as adding sparse component spikes of  signal. We wish to examine $\E[\norm{\B \x}^2_2]$. We examine a single component of the resulting vector:

\begin{gather}
\E \left[\left[\B \x \right]^2_i \right] = \E \left[\left(\sum_{j=1}^{N_2} \B_{i,j} \x_j \right)^2\right] = \E \left[ \sum_ {j, j'}\B_{i,j} \B_{i,j'}  \x_{j'} \x_j\right]\\
= \E \left[ \sum_{i} \B^2_{i,i} \x^2_i\right] = N_2 \E \left[ \int_{\tau}^{\infty} w^2(\alpha w_m^{\alpha}) w^{-(\alpha + 1)} dw \right]\cdot \x^2_i = N_2 \frac{\tau^{-(\alpha - 2) }(\alpha w_m^{\alpha})}{\alpha - 2} \cdot \x^2_i
\end{gather}

Thus, in expectation, each component of $\B \x$ will add a power-law factor in $\W (t)\x$. To refine this analysis, rather than computing the expected value, which can smooth the effects of the large matrix elements, we can also consider whether a given row of $\B$ (and hence component of $\B\x$) will contain a large matrix element using the same bounding argument as Proposition \ref{binomial}. Thus the probability that a given row of $\B$ will be $k$-sparse is bounded by:

\begin{equation}
\pr(X \geq k)  \leq \left(\frac{N_2w_m^{\alpha}\tau^{-\alpha}}{k}\right)^k \left(\frac{1-w_m^{\alpha}\tau^{-\alpha}}{1-\frac {k} {N_2}}\right)^{N_2-k}
\end{equation}

These ``spikes'' will be added the the compressed $\G \x$. 

\subsection{Resilient Classification}

It was found in the work of \cite{fortuin2021bayesian} that later layers in the network become progressively more heavy-tailed. We give a heuristic analysis of what benefits heavy-tailed behavior in the later layers may potentially give. We begin by defining a notion of \emph{resilient classification}. For a network with cross-entropy loss with a preceding softmax function, the cross-entropy loss is minimized when the softmax output for the true label $e^{z_i} \rightarrow \infty$, and the un-true labels $e^{z_j} \rightarrow 0$ for $j \neq i$.  By way of Proposition \ref{sherpa3}, with a shrinking value of $\alpha$, we can replace a growing number of the matrix elements with a Gaussian random variable without affecting the classification result. We refer to this phenomenon as \emph{resilient classification}. The following analysis is inspired by the spirit of the numerical renormalization group \cite{wilson1975renormalization}. 

By way of the aforementioned illustration, suppose $\W$ is the fully trained weight matrix of a neural network where the matrix elements $\W_{i,j}$ are drawn from an $\alpha$-stable distribution which behaves asymptotically like a power-law. Along with the stability parameter $\alpha$, they are parameterized by a scale factor $c$, skewness parameter $\beta$ (which measures the asymmetry), and shift parameter $\delta$ (which is roughly analogous to the mean). For simplicity, we assume $\beta =  \delta = 0$. The asymptotic behavior is given by (\cite{nolan2009univariate}):
 
 \begin{equation}
\pr(w | \alpha, \beta = 0, c, \delta = 0) \sim \alpha c^\alpha c_\alpha w^{-(\alpha + 1)}
 \end{equation}
 
 where $c_{\alpha} = \sin(\frac{\pi \alpha}{2})\Gamma(\alpha)/\pi$. Suppose we separate the matrix elements of $\W$ into powers of the scale factor $c$, i.e.:
 
 \begin{equation}
\W = \sum_{i=0}^M \C^i + \D
 \end{equation}
 
Where  $\C^i_{k,l} = \W_{k,l}\one[c^i < \W_{k,l} < c^{i+1}]$ and $\D_{k,l} = \W_{k,l}\one[c^{M+1} < \W_{k,l}]$. For large enough $c^i$, the probability of a given $\C^i_{k,l}$ being non-zero is given by:

\begin{equation}
p = c_{\alpha}\frac{1}{c^{\alpha (i-1)}} \left( 1 -  c^{-\alpha}\right) 
\end{equation}

Which is given by an integration between the bounds. For a given row, we can bound the probability that the row has a matrix element in a given ``power bracket" of $c$ with a binomial bound, as in Proposition \ref{binomial}. Suppose $p$ is as defined above. Let $\kappa$ be the number of non-zero elements, and the total number of elements in the row be $N$. Let $X$ be the number of non-zero elements. For $\kappa > Np$: 

\begin{equation}
\pr(X \geq \kappa)  \leq \left(\frac{Np}{\kappa}\right)^\kappa \left(\frac{1-p}{1-\frac \kappa N}\right)^{N-\kappa}
\end{equation}

We wish to analyze the action of $\W$ on $\x$ at the output layer, where for classification tasks a softmax function is typically applied. Suppose $\z = \W \x $. We pass this through a softmax function, $\text{softmax}(\z_i) = \frac{e^{\z_i}}{\sum_{i}e^{\z_i}}$.

\begin{figure}[ht]
    \centering
    \includegraphics[scale=0.50]{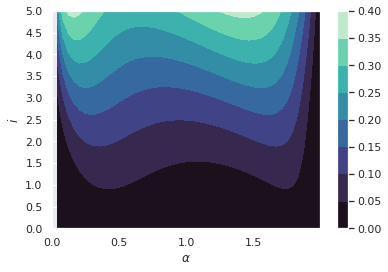}
    \caption{Contour plot of the probability as a function of $\alpha$ and $i$. The parameters of $c= 1.3, M=5, N=64$ are chosen based on the fully trained weight matrix of a ResNet. A lighter color indicates higher probability.}
    \label{contour}
\end{figure}

As stated above, the cross-entropy loss is minimized when the true label $\z_i \rightarrow \infty$ and the un-true labels $j \neq i, \z_j \rightarrow -\infty \, (e^{\z_j}\rightarrow 0)$. Suppose $\z_i$ of the true label is of magnitude $c^M$. One way to reach the value of $c^M$ is to have $c^{M - i}$ elements of magnitude $c^i$ in a given row (if we assume that the features are of order 1). For $\kappa = c^{M-i} > Np = Nc_{\alpha}\frac{1}{c^{\alpha (i-1)}} \left( 1 -  c^{-\alpha}\right) $:

\begin{equation}
\pr(X \geq c^{M-i})  \leq \left(\frac{Nc_{\alpha}\frac{1}{c^{\alpha (i-1)}} \left( 1 -  c^{-\alpha}\right) }{c^{M-i}}\right)^{(c^{M-i})} \left(\frac{1-c_{\alpha}\frac{1}{c^{\alpha (i-1)}} \left( 1 -  c^{-\alpha}\right) }{1-\frac{c^{M-i}} N}\right)^{N-c^{M-i}}
\end{equation}

As $\alpha$ grows smaller, we can replace more of the matrix elements smaller than some cut-off power $c^{\tau}$, while still leaving many ``paths" to attain $c^M$ in a given coordinate. In Figure \ref{contour}, we plot the probability as a function of $\alpha$ and $i$ for $c = 1.3, M = 5, N = 64$. For a fixed value of $\alpha$, we see that there are large regions of $c^{M-i}$ that are probable.

 
 

 \subsection{Stable Rank and Power-laws}

The stable rank of a matrix $\M$ is defined as $\text{srank}(\M) \equiv \frac{\norm{\M}^2_F}{\norm{\M}^2_2}$, where $\norm{\cdot}_F$ is the Frobenius norm and $\norm{\cdot}_2$ is the spectral norm. In the work of \cite{arora2018stronger}, of the four constants used in the induction proof, it was found that the layer cushion, $\mu_i$, was the most correlated with generalization.  The layer cushion was defined to be the largest number $\mu_i$ such that:
 
 \begin{equation}
 \mu_i \leq \frac{\norm{\W^i \phi(\x^{i-1})}}{\norm{\W}_F\norm{\phi(\x^{i-1})}}
 \label{layer-cushion}
 \end{equation}
 
where $\W^i$ is the weight matrix of layer $i$, $\phi$ is the activation function, and $\x^{i-1}$ is the output of the $(i-1)$-th layer before the activation function. It was found that a larger layer cushion (smaller inverse layer cushion) was correlated with generalization. From equation \ref{layer-cushion}, we can infer that  the inverse layer cushion is lower-bounded by $\sqrt{\text{srank}(\W)} \leq \frac{1}{\mu_i}$. In general, we know that the Frobenius norm bounds the spectral norm $\norm{\W}_2 \leq \norm{\W}_F$. 
 
 In the work of \cite{rebrova2018spectral}, lower bounds on the spectral norm of a heavy-tailed random matrix (with matrix elements with diverging variance) are given. A rigorous treatment of the relationship between the stable rank and power-law distributed matrix elements remains an open question. Heuristically, as the power-law exponent $\alpha$ grows smaller, the Frobenius norm of a matrix with heavy-tailed matrix elements will grow, as will the spectral norm. In Figure \ref{stable-rank}, we numerically compute the stable-rank as a function of $\alpha$ for the weight matrices of several architectures. The spectral norm of the weight matrix is computed through power iteration. The fit of the matrix elements to a power-law distribution with exponent $\alpha$ is done with the \texttt{powerlaw} package \cite{alstott2014powerlaw}, which uses a maximum-likelihood fit. The full details of the computation and list of architectures is given in Appendix \ref{appendixD}. In Figure \ref {stable-rank}, we see there is a general clustering of smaller $\alpha$ values with smaller stable rank, which is fit to a mixture of linear regressors using expectation maximization. 
 
\begin{figure}
    \centering
    \includegraphics[scale=0.55]{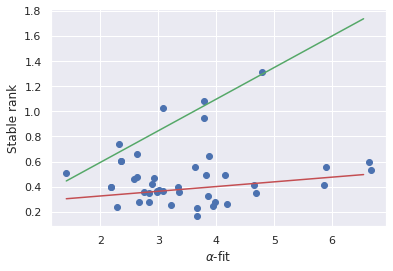}
    \caption{Stable rank of the fully trained weight matrices of a neural network as a function of $\alpha$. We see a general clustering of a lower value of $\alpha$ with a smaller stable rank, which is fit to a mixture of linear regressors using expectation maximization.}
    \label{stable-rank}
\end{figure}
 

\section{Empirical Evaluation}
 
In this section, the matrix compression algorithm is evaluated on several pre-trained neural network architectures, and the resulting accuracy of the compressed network will be compared to the original network. In Table \ref{accuracy}, we have a table of the model accuracy before and after the matrix compression of the final dense matrix, along with an $\alpha$-fit of the matrix elements (using the \texttt{powerlaw} package), as well as the dataset that the model was trained on. Compression for each model is performed 10 times and averaged, with standard deviation computed from the 10 trials. The full experimental details are given in Appendix \ref{appendixE}. As apparent from the table, compressing the final dense layer results in an accuracy comparable to the original model. 

\begin{table}[h]
  \caption{Model Accuracy Pre- and Post-Compression}
  \label{sample-table}
  \centering
  \begin{tabular}{lllll}
    \toprule
    \multicolumn{2}{c}{Models}                   \\
    \cmidrule(r){1-2}
    Architecture     & Dataset & $\alpha$-fit     & Original Accuracy & Compressed Accuracy \\
    \midrule
    densenet40\_k12\_cifar10 \cite{huang2017densely} & CIFAR10 & 2.84 & 94.46\% & 94.26 $\pm$ 0.09\\
    resnet20\_cifar10 \cite{he2016deep} & CIFAR10  & 3.22  & 93.90\% & 89.19 $\pm$3.03   \\
    resnet56\_cifar10 & CIFAR10 & 3.02  & 95.32\% & 94.74 $\pm$0.30  \\
    resnet110\_cifar10 & CIFAR10 & 2.98  & 96.26\% & 95.63 $\pm$ 0.63   \\
    resnet164bn\_cifar10 & CIFAR10 & 4.65  & 96.24\% & 95.78 $\pm$ 0.15   \\
    resnet272bn\_cifar10 & CIFAR10 & 2.89  & 96.63\% & 96.51 $\pm$ 0.06   \\
    densenet40\_k12\_svhn & SVHN & 1.39 & 96.99\%  & 96.99 $\pm$ 0.02\\
    xdensenet40\_2\_k24\_bc\_svhn & SVHN & 2.99 & 97.08\% & 96.97 $\pm$ 0.05 \\
    resnet20\_svhn & SVHN & 3.34 & 96.60\%  & 94.60 $\pm$ 0.71 \\
    resnet56\_svhn & SVHN & 2.63 & 97.19\%  & 96.58 $\pm$ 0.39 \\
    resnet110\_svhn & SVHN & 2.33 & 97.48\%  & 97.16 $\pm$0.14 \\
    pyramidnet110\_a48\_svhn \cite{han2017deep} & SVHN & 3.64 & 97.56\% & 97.28 $\pm$ 0.11 \\
    \bottomrule
  \end{tabular}
  \label{accuracy}
\end{table}

\section{Conclusion and Future Work}

In this work, we have shown how matrices with heavy-tailed matrix elements can be compressed using the compression framework, resulting in a generalization bound which counts the number of non-zero matrix elements in sparse matrices. In addition, we have analyzed how the early layers in an FCNN may be acting as compression layers, and how later layers may be building \emph{resilient classification}. We have also discussed the relationship between the stable-rank of a matrix and matrix-elements that are heavy-tailed. Finally, we have provided empirical evidence that the matrix approximation algorithm results in networks with comparable accuracy. 

For future work, an important direction is relaxing the i.i.d. assumption of the matrix elements of the trained weight matrix, and building models of correlation which recover the power-law behavior in the matrix elements. In connection with this direction, the relationship between the eigenvalue distribution of the covariance matrix of the weights and the matrix elements themselves can be further explored theoretically. Finally, the relationship between the stable-rank of a matrix and power-law distributed matrix elements could also be rigorously established. 

\section{Acknowledgements}
This work was supported by the \emph{David K.A. Mordecai \& Samantha Kappagoda Charitable Trust Graduate Student Research Fellowship}. This research was conducted in residence at the \emph{RiskEcon® Lab, Courant Institute of Mathematical Sciences NYU} with funding from \emph{Numerati® Partners LLC}. The author would like to thank David K.A. Mordecai for extensive comments throughout the manuscript. The author would also like to thank Joseph P. Aylett-Bullock for comments on an early version of the manuscript. The author also thanks Charles Martin and Michael Mahoney for discussion of their work.


 
 
 




\bibliography{main}
\bibliographystyle{unsrt}

\appendix
\section{Bounding the Error}
\label{appendixA}
Let $\W$ be a matrix of of arbitrary dimension $N_1 \times N_2 $ , where each matrix element is drawn $\W_{i,j} \sim (\alpha w_m^{\alpha})w^{-\alpha}, i \in \{1, \cdots, N_1\}, j \in \{1, \cdots N_2\}$ from a power-law distribution with support $w \in [w_m, \infty)$, where $w_m > 0$. Suppose for some threshold $\tau > w_m$, we separate the matrix elements less than or equal to $\tau$ in a matrix $\m{A}$ with the same index that it had in $\W$, and likewise the matrix elements greater than $\tau$ we separate into the matrix $\B$ at the same index that it had in $\W$. Thus, $\W = \A + \B$. We wish to replace the matrix $\A$ with a matrix $\sqrt{t}\m{G}$, a matrix with i.i.d normal distribution entries scaled by $\sqrt{t}$. We define the matrix $\W(t) = \sqrt{t}\G + \B$. We wish to track the error introduced by replacing $\W$ with $\W(t)$ in a neural network. Let $\overline{\W} = \W - \W(t) = \A - \sqrt{t}\G$. 

\begin{proposition}
For any $\m{u} \in \R^{N_1}, \m{v} \in \R^{N_2}$, with probability $1 - 3\exp\left(- \frac{s^2}{2\norm{\m{uv}^T}^2_F(\tau^2+t)}\right)$, \\
$$\left| \m{u}^T \overline{\W} \m{v} \right| < s$$
\label{sherpa}
\end{proposition}
\begin{proof}
We know that $\A_{i,j} = \W_{i,j}\one(\W_{i,j} < \tau)$. Let $\A_{i,j}$ be the $i,j$-th matrix element of $\A$, and $\G_{i,j}$ be the $i,j$-th matrix element of $\G$. Let $\m{v}_i$ and $\m{u}_i$ be the $i$th component of the vectors $\m{v}$ and $\m{u}$ respectively.
\begin{gather}
\pr \left\{\m{u}^T \overline{\W} \m{v} > s \right\} = \pr \left\{\m{u}^T (\A - \sqrt{t} \G)\m{v} > s\right\} \\
= \pr \Big\{ \m{u}_1\left[(\A_{1,1}- \sqrt{t}\G_{1,1})\m{v}_1 + \cdots (\A_{1 N_2}- \sqrt{t}\G_{1, N_2})\m{v}_{N_2} \right] + \cdots + \\
\m{u}_{N_1}\left[(\A_{N_1,1} - \sqrt{t}\G_{N_1,1})\m{v}_1 + \cdots + (\A_{N_1,N_2} - \sqrt{t}\G_{N_1 ,N_2})\m{v}_{N_2} >s \right] \Big\} \\
\begin{align}
 &= \pr \left\{ \sum_{ij}\m{u}_i (\A_{ij} - \sqrt{t}\G_{ij})\m{v}_j > s \right\} \\
 &\leq \frac{\E\left[  \sum_{i,j}\m{u}_i (\A_{i,j} - \sqrt{t}\G_{i,j})\m{v}_j\right]}{s} \text{ (Markov).}\\
 &\leq \frac{\E\left[ \exp\left( \lambda \sum_{i,j}\m{u}_i (\A_{i,j} - \sqrt{t}\G_{i,j})\m{v}_j \right)\right]}{e^{\lambda s}} \text{ (Convexity, Jensen's).}\\
\end{align}
\end{gather}

For ease of notation, let $a_{ij} = \m{u}_i \m{v}_j$ and $g_{ij} = \sqrt{t}\m{u}_i \m{v}_j$. We focus on the right hand side.

\begin{gather}
\begin{align}
&= \frac{\E\left[ \exp\left( \lambda \sum_{i,j} a_{ij}\A_{i,j} - g_{ij}\G_{i,j} \right)\right]}{e^{\lambda s}}\\
&\leq \frac{\E\left[ \exp\left( \lambda \sum_{i,j} a_{ij}\tau - g_{ij}\G_{i,j} \right)\right]}{e^{\lambda s}} \text{ (Threshold)}\\
&\leq \frac{\E\left[ \exp\left( \lambda \sum_{i,j} a_{ij}\tau + g_{ij}\G_{i,j} \right)\right]}{e^{\lambda s}} \text{ (Symmetry)}
\end{align}
\end{gather}

We separate the exponential. The first sum is no longer random. 

\begin{equation}
= \frac{\exp\left( \lambda \sum_{i,j} a_{ij}\tau \right) \E \left[ \exp \left(  \lambda \sum_{i,j} g_{ij}\G_{i,j} \right) \right]}{e^{\lambda s}} 
\end{equation}

The expected value in the numerator is merely the moment generating function of a normal distribution. From a classical result, $\E[\exp(\lambda X)] = e^{\lambda^2/2}$ when $X \sim \N(0,1)$. We also bound the first exponential with $e^x \leq e^{1+x^2/2}, \forall x \in \R$ for convenience. 

\begin{gather}
\leq \frac{\exp\left( 1 + \lambda^2/2\sum_{i,j} a_{ij}^2 \tau^2\right)\exp( \lambda^2\sum_{i,j} g^2_{ij}/2)}{e^{\lambda s}} \\
= \exp\left(-\lambda s + 1 + \lambda^2/2 \sum_{i,j} a^2_{ij}\tau^2 + \lambda^2 g^2_{ij}/2\right)
\end{gather}

We optimize over $\lambda$. $\frac{d}{d\lambda} e^{- \lambda s + 1 + \lambda^2 a + \lambda^2 b} = (2\lambda a + 2\lambda b - s)e^{- \lambda s + \lambda^2 a + \lambda^2 b} = 0, \lambda = \frac{s}{2(b+a)}$. Let $a = 1/2\sum_{i,j} a^2_{ij}\tau^2$ and $b = 1/2 \sum_{i,j}g^2_{ij}$ for convenience. Putting this back into our expression, we have:

\begin{gather}
\begin{align}
&= \exp\left(- \frac{s^2}{2(a+b)} + 1 + \frac{s^2}{4(a+b)} \right) \\
&= \exp\left(- \frac{s^2}{4(a+b)} + 1 \right) \\
&\leq 3 \exp\left(- \frac{s^2}{4(a+b)}\right) \text{ (Pull out $e$)}\\
&= 3 \exp\left(- \frac{s^2}{2(\sum_{i,j}a^2_{ij}\tau^2 + g_{ij}^2)}\right) \text{ (Replace substitution)}\\
&= 3 \exp\left(- \frac{s^2}{2(\sum_{i,j}\m{u}^2_i\m{v}_j^2(\tau^2 + t))}\right) \text{ (Replace substitution)}\\
&= 3 \exp\left(- \frac{s^2}{2\norm{\m{u}\m{v}^T}^2_F (\tau^2 + t)}\right) \text{ (Frob. norm)}
\end{align}
\end{gather}

Thus our final statement is:

\begin{equation}
\pr \left\{\m{u}^T \overline{\W} \m{v} > s \right\} \leq 3 \exp\left(- \frac{s^2}{2\norm{\m{u}\m{v}^T}^2_F (\tau^2 + t)}\right)
\end{equation}
By symmetry, we obtain the result for the absolute value. For $s = \epsilon \norm{uv^T}_F$ and $\frac{2}{\tau^2 + t} = \frac{\log(3/\eta)}{\epsilon^2}$, we prove the proposition.
\end{proof}

\section{Sparsity}
\label{appendixB}

Let $\B_{k,l}$ be the $k,l$ matrix element of $\B$, $\B_{k,l} = \W_{k,l}\one[\W_{k,l} > \tau]$, $k \in \{ 1, \cdots, N_1\}, l \in \{1, \cdots, N_2\}$, where $\W_{k,l} \sim (\alpha w_m^{\alpha})w^{-(\alpha+1)}$ is drawn from a power law distribution with exponent $\alpha$, and  with support $w \in [w_m, \infty)$, where $w_m > 0$. The probability that any given matrix element of $\B$ is non-zero is given by:

\begin{gather}
\int_{\tau}^{\infty}p(w) dw = \int_{\tau}^{\infty}(\alpha w_m^{\alpha })w^{-(\alpha+1)} dw = w_m^{\alpha}\tau^{-\alpha}
\end{gather}

For $\alpha > 0, \tau > w_m > 0$. So with probability $ w_m^{\alpha}\tau^{-\alpha}$, a given matrix element is non-zero, and with probability $1 -  w_m^{\alpha}\tau^{-\alpha}$, the matrix element is zero. This is a Bernoulli random variable with probability $p =  w_m^{\alpha}\tau^{-\alpha}$. The probability that the number of non-zero matrix elements is less than or equal to $k$ is given by the cumulative distribution function of the binomial distribution:

\begin{equation}
\pr(X \leq k) = \sum_{i=0}^{\floor{k}}{n\choose i} p^i (1-p)^{n-i} 
\end{equation}

Where $X$ is the number of non-zero elements, $p$ is the probability of a non-zero element, and $n$ is the total trials. Suppose $n = N_1 \times N_2$. Let $k$ be our chosen sparsity condition. We have:

\begin{equation}
\pr(X \leq k) = \sum_{i=0}^{\floor{k}}{n\choose i} \left(\frac{\tau^{-(\alpha )}}{\alpha}\right)^i \left(1-\frac{\tau^{-(\alpha )}}{\alpha}\right)^{n-i} 
\end{equation}

For $k/n > p$, the right tail can be bounded with a Chernoff bound to give:

\begin{equation}
\pr( X \geq k) \leq \exp \left(-nD\left(\frac kn \middle\| p \right)\right)
\end{equation}

where $D$ is the relative entropy between the first and second distribution:

\begin{equation}
D \left( a \middle\| b \right) = (a) \log \frac ap + (1-a) \log\left(\frac{1-a}{1-p}\right)
\end{equation}

Plugging in our terms, we have:

\begin{gather}
\begin{align}
\pr( X \geq k) &\leq \exp \left(- n \left(\frac kn \log \frac{k}{np} + \left(1 - \frac kn\right)\log\left(\frac{1-k/n}{1-p}\right)\right)
\right)\\
&=\exp \left(- k \log \frac{k}{np} - \left(n - k\right)\log\left(\frac{1-k/n}{1-p}\right)\right)\\
&=\exp \left( \log \left(\left(\frac{np}{k}\right)^k \right) + \log\left( \left(\frac{1-p}{1-k/n}\right)^{n-k}\right)\right)\\
&=\exp \left( \log \left(\left(\frac{np}{k}\right)^k \left(\frac{1-p}{1-k/n}\right)^{n-k} \right)\right)\\
&=\left(\frac{np}{k}\right)^k \left(\frac{1-p}{1-k/n}\right)^{n-k}\\
&=\left(\frac{N_1 \cdot N_2 w_m^{\alpha}\tau^{-\alpha}}{k}\right)^{k} \left(\frac{1- w_m^{\alpha}\tau^{-\alpha}}{1-\frac{k}{N_1 \cdot N_2}}\right)^{N_1 \cdot N_2 -k}
\end{align}
\end{gather}

\section{Proof of Theorem 2}
\label{appendixC}

Suppose we have a fully-connected neural network with $d$ layers, where the weight matrix for each layer is indexed by $\W^i, \quad i \in \{ 1, \cdots, d\}$. Let $\phi$ be the activation function. For each weight matrix, let the dimension of each be given by $\W^i \in \R^{h_i \times h_{i-1}}$. Where $h^0 \in \R^b$ is the input dimension. Let $\m{x}^i$ denote the output of the $i$'th layer before the activation function, that is: $\W^i \phi(\m{x}^{i-1})$. Let $h = \max_{i=1}^L h_i$ be the maximum hidden dimension. Let $\m{J}$ be the Jacobian of the full neural network, and $\m{J}^{ij}$ be the interlayer Jacobian from layer $i$ to layer $j$. Likewise, let $\m{M}^{ij}$ be the composition of the layer $i$ to layer $j$. Let $m$ be the cardinality of the training dataset, $S$.

For each layer matrix $\W^i$, suppose each matrix element is drawn $\W^i_{k,l} \sim (\alpha w_m^{\alpha})w^{-(\alpha+1)}, k \in \{1, \cdots, h_i\}, l \in \{1, \cdots h_{i-1}\}$ from a power-law distribution with power law exponent $\alpha$ and with support $w \in [w_m, \infty)$, where $w_m > 0$. Suppose for some threshold $\tau_i$, we separate the matrix elements less than or equal to $\tau_i$ in a matrix $\m{A}^i$ with the same index that it had in $\W^i$, and likewise the matrix elements greater than $\tau_i$ we separate into the matrix $\B^i$ at the same index that it had in $\W^i$. Thus, $\W^i = \A^i + \B^i$. We wish to replace the matrix $\A^i$ with a matrix $\sqrt{t_i}\m{G}$, a matrix with i.i.d normal distribution entries scaled by $\sqrt{t_i}$. We define the matrix $\W(t_i)^i = \sqrt{t_i}\G + \B$. Let $\overline{\W^i} = \W^i - \W(t_i)^i = \A^i - \sqrt{t_i}\G$.

The proof strategy is largely similar to the one in \cite{arora2018stronger}. We cite the lemmas we reference, and note when we have modified them. We first prove the following lemma, which bounds the error in the approximation with a matrix on the left-hand side rather than a vector. This will be used to prove another lemma.

\begin{lemma} (\cite{arora2018stronger})
For any $0 < \delta, \epsilon \leq 1$, Let $\m{G} = \{(\m{U}^i, \m{x}^i)^m_{i=1} \}$ be a set of matrix/vector pairs of size $m$ where $\m{U} \in \R^{n \times h_i}$ and $\m{x} \in \R^{h_{i-1}}$. Let $\W$ and $\W(t)$ be weight matrices as defined above with the appropriate dimension, with $\eta  = \delta/(mn)$. With probability $1 - \delta$ we have for any $(\m{U}, \m{x}) \in \m{G}$, $\norm{\m{U}(\W-\W(t))\m{x}} \leq \epsilon \norm{\m{U}}_F\norm{\m{x}}$.
\end{lemma}

\begin{proof}

Let $\m{u}_i$ be the $i$-th row of $\m{U}$. Using Proposition \ref{sherpa}, let us choose $s = \epsilon \norm{\m{u}\m{v}^T}_F =  \epsilon \norm{\m{u}}\norm{\m{v}^T}$.

For all $\m{u}_i$ we have that, 

\begin{equation}
\pr \left\{ |\m{u}_i^T\overline{\W}v|  \geq \epsilon \norm{\m{u}_i} \norm{\m{v}} \right\} \leq  \eta
\end{equation}

Since $\norm{\m{U}(\W - \W(t))\m{x}}^2 = \sum_i \norm{\m{u}_i(\W - \W(t))\m{x}}^2$ and $\norm{\m{U}}^2_F = \sum_i \norm{\m{u}_i}^2$, by union bound, we have our result.
\end{proof}

We define several the quantities we need for the next lemma, the layer cushion, the interlayer cushion (and minimal interlayer cushion), the activation contraction, and the interlayer smoothness, as defined in \cite{arora2018stronger}.

\begin{enumerate}
    \item \textbf{Layer cushion} ($\mu_i$) For any layer $i$, we define the layer cushion $\mu_i$ as the largest number such that for any $x \in S$:
    \begin{equation}
    \mu_i \norm{\W^i}_F \norm{\phi(\x^{i-1})} \leq \norm{\W^i \phi(\x^{-1})}
    \end{equation}
    \item \textbf{Interlayer cushion} ($\mu_{ij}$): For any two layers $i \leq j$, we define the interlayer cushion $\mu_{ij}$ as the largest number such that for any $x \in S$:
    \begin{equation}
    \mu_{ij}\norm{\m{J}^{ij}_{\x^i}}_F \norm{\x^i} \leq \norm{\J^{ij}_{\x^i}\x^i}
    \end{equation}
    The minimal interlayer cushion is defined:\\
    \begin{equation}
    \mu_{i \rightarrow} = \min_{i \leq j \leq d} \mu_{ij} = \min\{1/\sqrt{h^i}, \min _{i< j \leq d}\mu_{ij} \}
    \end{equation}
    \item \textbf{Activation contraction} ($c$): The activation contraction $c$ is defined as the smallest number such that for any layer $i$ and any $x \in S$:
    \begin{equation}
    \norm{\x^i} \leq c \norm{\phi(\x^i)}
    \end{equation}
    \item \textbf{Interlayer smoothness} ($\rho_{\delta}$): Interlayer smoothness is defined as the smallest number such that with probability $1 - \delta$ over noise $\etab$ for any two layers $i < j$ and any $x \in S$:
    \begin{equation}
    \norm{\m{M}^{ij}(\x^i + \etab) - \J^{ij}_{\x^i}(\x^i + \etab)} \leq \frac{\norm{\etab}\norm{\x^j}}{\rho_{\delta}\norm{\x^i}} 
    \end{equation}
\end{enumerate}

We wish to prove a new lemma which bounds the total error of replacing all the weight matrices with the Gaussian approximation. This is done through an inductive argument. 

\begin{lemma} (Modified from \cite{arora2018stronger})
For any fully connected network $f_{\W}$ with $\rho_{\delta} \geq 3L$, and any error $0 < \epsilon \leq 1$, the Gaussian approximation generates weights $\W(t)$ for a network, such that with probability $1 - \frac{\delta}{2}$ over the generated weights $\W(t)$ for any $x \in S$:

\begin{equation}
\norm{f_{\W}(\x) - f_{\W(t)}(\x)} \leq \epsilon \norm{f_{\W}(\x)}
\end{equation}
\label{induction}
\end{lemma}

\begin{proof}
We will prove this by induction. For any layer $i \geq 0$, let $\hat{\x}^j_i$ be the output at the layer $j$ if the weights $\W^1, \cdots, \W^i$ in the first $i$ layers are replaced with $\W^1(t), \cdots, \W^i(t)$. The induction hypothesis is then given by:

Consider any layer $i \geq 0$ and any $0 < \epsilon \leq 1$. The following is true with probability $1 - \frac{i\delta}{2d}$ over $\W^1(t), \cdots, \W^i(t)$ for any $j \geq i$:

\begin{equation}
\norm{\hat{\x}^j_i -\x^j} \leq (i/d)\epsilon \norm{\x^j}
\end{equation}

For the base case $i=0$, since we are not perturbing the input, the difference is zero and hence the bound holds. Now, for the induction hypothesis, we assume the $(i-1)$ case is true, and demonstrate the $i$ case. 

\begin{gather}
\norm{\hat{\x}^j_{i-1} - \x^j} \leq \frac{(i-1)\epsilon}{d}\norm{\x^j} \text{ (Induction hypothesis).}\\
\norm{\hat{\x}^j_{i} - \x^j} \leq \frac{i\epsilon}{d}\norm{\x^j} \text{ (To show).}
\end{gather}

For each layer, we choose epsilon $\epsilon_i = \frac{\epsilon \mu_{i}\mu_{i \rightarrow}}{6cd}$ and $\eta_i = \frac{\delta}{2(dm)(h)(d)}$ (recall that we needed to choose $\eta = \frac{\delta}{m n}$ to use the lemma, where $m$ was the cardinality of the set and $n$ the first dimension of the matrix). 
We can apply our lemma to the set $G = \{(\J^{ij}_{\x^i}, \x^i) | \x \in S, j \geq i \}$, which has size at most $dm$. Let $\Delta^i = \W^i - \W^i(t)$ for any $j \geq i$:

\begin{equation}
\norm{\hat{\x}^j_i - \x^j} = \norm{\hat{\x}^j_i - \hat{\x}^j_{i-1} + \hat{\x}^j_{i-1} - \x^j} \leq \norm{(\hatx{j}{i} - \hatx{j}{i-1})} + \norm{\hatx{j}{i-1} - \x^j}
\end{equation}

The first equality adds zero, and the second inequality is the triangle inequality. The second term can be bounded by $(i-1)\epsilon \norm{\x^j}/d$ by the induction hypothesis. If we can bound the first term by $\epsilon\norm{\x^j}/d$ we prove the induction. We decompose this error in two terms:

\begin{gather}
\norm{(\hatx{j}{i} - \hatx{j}{i-1})} = \norm{\M^{ij}\W^i(t)\phi(\hatx{i-1}{})- \M^{ij}\W^i \phi(\hatx{i-1}{})}  \\
= \norm{\M^{ij}\W^i(t)\phi(\hatx{i-1}{}) - \M^{ij}\W^i(t)\phi(\hatx{i-1}{}) + \J^{ij}_{\x^i}\Delta^i \phi(\hatx{i-1}{}) - \J^{ij}_{\x^i}\Delta^i\phi(\hatx{i-1}{})} \label{breakup}\\
\leq \norm{\J^{ij}_{\x^i}\Delta^i \phi(\hatx{i-1}{})} + \norm{\M^{ij}\W^i(t)\phi(\hatx{i-1}{}) - \M^{ij}\W^i(t)\phi(\hatx{i-1}{})  - \J^{ij}_{\x^i}\Delta^i\phi(\hatx{i-1}{})}
\end{gather}

The first equality is results from the fact that all terms cancel but these in the difference. The second equality merely adds zero. The last inequality results from the triangle inequality. We assume $\norm{\W^i}_F \geq 1$. We note that we can remove the dependence on the interlayer smoothness if we assume a ReLU network, since $\M^{ij} = \J^{ij}$ and thus we can skip the decomposition in \ref{breakup} and only perform the bound on the Jacobian. 

We first focus on the first term. 

\begin{gather}
\begin{align}
&\norm{\J^{ij}_{\x^i}\Delta^i \phi(\hatx{i-1}{})}\\
&\leq \frac{\epsilon \mu_i \mu_{i \rightarrow}}{6cd}\norm{\J^{ij}_{\x^i}}\norm{\phi(\hatx{i-1}{})} \text{ (Lemma)} \\
&\leq \frac{\epsilon \mu_i  \mu_{i \rightarrow}}{6cd}\norm{\J^{ij}_{\x^i}}\norm{\hatx{i-1}{}} \text{ (1-Lipschitz Activation)} \\
&\leq \frac{\epsilon \mu_i  \mu_{i \rightarrow}}{3cd}\norm{\J^{ij}_{\x^i}}\norm{\x^{i-1}} \text{ (Induction Hypothesis)} \\
&\leq \frac{\epsilon \mu_i  \mu_{i \rightarrow}}{3d}\norm{\J^{ij}_{\x^i}}\norm{\phi(\x^{i-1})} \text{ (Activation Contraction)} \\
&\leq \frac{\epsilon  \mu_{i \rightarrow}}{3d}\norm{\J^{ij}_{\x^i}\phi(\x^{i-1})} \text{ (Layer Cushion)} \\
&= \frac{\epsilon  \mu_{i \rightarrow}}{3d}\norm{\J^{ij}_{\x^i}\x^{i}} \text{ (Definition of $\x^i$)} \\
&\leq \frac{\epsilon}{3d}\norm{\x^{j}} \text{ (Interlayer Cushion)}
\end{align}
\end{gather}

Next, we go back to the second term we wished to bound. We need to bound this term by $\frac{2\epsilon}{3d}\norm{\x^j}$ to get our final $\epsilon/d\norm{\x}^j$. 

\begin{gather}
\norm{\M^{ij}\W^i(t)\phi(\hatx{i-1}{}) - \M^{ij}\W^i(t)\phi(\hatx{i-1}{}) +  - \J^{ij}_{\x^i}\Delta^i\phi(\hatx{i-1}{})}\\
= \norm{(\M^{ij} - \J^{ij}_{\x^i})\W^i(t)\phi(\hatx{i-1}{}) - (\M^{ij} - \J^{ij}_{\x^i})\W^i(t)\phi(\hatx{i-1}{})}\\
\leq \norm{(\M^{ij} - \J^{ij}_{\x^i})\W^i(t)\phi(\hatx{i-1}{})} + \norm{ (\M^{ij} - \J^{ij}_{\x^i})\W^i\phi(\hatx{i-1}{})}\\
\end{gather}
We start with the second term. We note that $\W^i\phi(\hatx{i-1}{}) = \hatx{i}{i-1}$ by definition. By induction hypothesis, $\norm{\W^i \phi(\hatx{i-1}{}) - \x^i} \leq (i-1)\epsilon\norm{\x^i}/d \leq \epsilon \norm{\x^i}$. By interlayer smoothness, we have:
\begin{equation}
\norm{ (\M^{ij} - \J^{ij}_{\x^i})\W^i\phi(\hatx{i-1}{})} \leq \frac{\norm{\x^j}\epsilon}{\rho_{\delta}} \leq \epsilon/(3d)\norm{\x^j}
\end{equation}

Likewise, we have $\W(t)\phi(\hatx{i-1}{}) = \hatx{i}{i-1} + \Delta^i\phi(\hatx{i-1}{})$ (adding zero). thus:

\begin{gather}
\norm{\W^i(t)\phi(\hatx{i-1}{})-\x^i} \\
\leq \norm{\W^i\phi(\hatx{i-1}{})-\x^i} \norm{\Delta^i\phi(\hatx{i-1}{})} \\
\leq (i-1)\epsilon\norm{\x^j}/d + \epsilon/(3d)\norm{\x^j} \leq \epsilon \norm{\x^j}
\end{gather}

Thus by interlayer smoothness we have:
\begin{equation}
\norm{ (\M^{ij} - \J^{ij}_{\x^i})\W^i(t)\phi(\hatx{i-1}{})} \leq (\epsilon/(3d))\norm{\x^j}.
\end{equation}

Combining all results, we have:
\begin{equation}
\norm{\hat{\x}^j_{i} - \x^j} \leq \frac{i\epsilon}{d}\norm{\x^j} 
\end{equation} 

For $i = j = d$, we have:

\begin{equation}
\norm{\hat{\x}^d_{d} - \x^d} = \norm{f_{\W}(\x) - f_{\W(t)}(\x)} \leq \epsilon \norm{\x^j} = \epsilon \norm{f_{\W}(\x)} 
\end{equation} 
Which completes the proof. 
\end{proof}

Next, we bound the empirical classification loss of the compressed network with the empirical margin loss of the original network, using the lemma just proved.

\begin{lemma} (Modified from \cite{arora2018stronger})
For any fully-connected network $f_A$ with $\rho_{\delta} \geq 3d$, and probability $0 < \delta \leq 1$ and any margin $\gamma > 0$, $f_A$ can be compressed to another fully connected network $f_{\tilde{A}}$ such that for any $\x \in S$, $\hat{L}_0(f_{\tilde{A}}) \leq \hat{L}_{\gamma}(f_A)$. 
\label{loss}
\end{lemma}

\begin{proof}
If $\gamma^2 > 2 \max_{\x \in S} \norm{f_A(\x)}^2_2$, for any pair $(\x,y)$ in the training set we have:
\begin{gather}
|f_A(\x)[y] - \max_{j \neq y} f_A(\x)[j]|^2 \leq 2 \max_{\x \in S} \norm{f_A(\x)}^2_2 < \gamma^2 \\
|f_A(\x)[y] - \max_{j \neq y} f_A(\x)[j]| < \gamma 
\end{gather}

Recall that the empirical classification loss $(\gamma = 0)$ is $\frac 1m \sum_{i_1}^m \one[ f(\x_i)[y_i] \leq \max_{j \neq y}f(\x)[j]]$. This will be at most $1$. Since the margin cannot be greater than $\gamma$, this means $\hat{L}_{\gamma}(f_A) = \frac 1m \sum_{i_1}^m \one[ f(\x_i)[y_i] \leq \gamma  + \max_{j \neq y}f(\x)[j]] = 1$. Thus, if $\gamma^2 > 2\max_{\x \in S}\norm{f_A(\x)}^2_2$, then $\hat{L}_0(f_{\hat{A}}) = \hat{L}_{\gamma}(f_A)$ and the lemma statement is true. If $\gamma^2 \leq 2 \max_{\x \in S}\norm{f_A(\x)}^2_2$, we can set $\epsilon^2 = \frac{\gamma^2}{2\max_{\x\in S}\norm{f_A(\x)}^2_2}$ in Lemma \ref{induction}. Thus we infer that for any $\x \in S$

\begin{equation}
\norm{f_A(\x) - f_{\tilde{A}}(\x)}_2 \leq \gamma / \sqrt{2}
\end{equation}

For any $(\x,y)$, if the margin loss $\hat{L}_{\gamma}(f_A)$ is $1$, then the inequality holds, since the left-hand side cannot be greater by definition. If the margin loss is less than one, this means that the output margin in $f_A$ is greater than $\gamma$ ($f_A(\x)[y] - \max_{j \neq y} f_A(\x)[j] > \gamma $). In order for the classification loss on the left hand side to be $1$, we need $f_{\tilde{A}}(\x)[y] - \max_{j \neq y}f_{\tilde{A}}(\x)[j] \leq 0$ to be true.  Which is possible when:

\begin{equation}
\norm{f_A(\x) - f_{\tilde{A}}(\x)}_2 > \gamma/\sqrt{2}
\end{equation}

Which is a contradiction. We demonstrate this as follows: let the $(y, \max_j)$ components of $f_A(\x)$ be $(\gamma + a, b)$ where $a > b$. Let the two $(y, \max_j)$ components of $f_{\tilde{A}}$ be $(c, d)$, where $d > c$. The squared 2-norm of these components is:

\begin{equation}
(\gamma + a - c)^2 + (b-d)^2
\end{equation}

If $a = c$, this is equal to $\gamma^2$ and hence the 2-norm is $\gamma > \gamma /\sqrt{2}$. If $a > c$, then the square of the two norm is $(\gamma + pos)^2 + (b-d)^2 > \gamma^2$, so the two norm is $\gamma > \gamma/\sqrt{2}$. If $a < c$, we expand:

\begin{gather}
\begin{align}
& (\gamma + a - c)^2 + (b-d)^2 \\
&=\gamma^2 + 2\gamma(a - c) + (a-c)^2 + (b-d)^2 \\
&=\frac{\gamma^2}{2} + \frac{\gamma^2}{2} + 2\gamma(a - c) + (a-c)^2 + (b-d)^2 \\
&=\frac{\gamma^2}{2} + \frac 12 (\gamma^2 + 4\gamma(a - c)) + (a-c)^2 + (b-d)^2 \\
&=\frac{\gamma^2}{2} + \frac 12 \left(\left(\gamma + \frac{4\gamma(a - c)}{2}\right)^2-\left(\frac{4(a - c)}{2}\right)^2\right) + (a-c)^2 + (b-d)^2 \\
&=\frac{\gamma^2}{2} + \frac 12\left(\gamma + \frac{4\gamma(a - c)}{2}\right)^2- \frac 12\left(\frac{4(a - c)}{2}\right)^2 + (a-c)^2 + (b-d)^2 \\
&=\frac{\gamma^2}{2} + \frac 12\left(\gamma + \frac{4\gamma(a - c)}{2}\right)^2- (a-c)^2 + (b-d)^2 \\
\end{align}
\end{gather}
We know $b < a < c < d$. Thus $(b-d)^2 > (a-c)^2$. So the square of the 2-norm is greater than $\gamma^2/2$ and the 2-norm is greater than $\gamma/\sqrt{2}$, which completes the proof. 

\end{proof}

We will utilize a lemma found in \cite{neyshabur2017pac}. 

\begin{lemma} (Originally from \cite{neyshabur2017pac} and also in \cite{arora2018stronger})
Let $f_A$ be a $d$-layer network with weights $A = \{ A^1, \cdots, A^d \}$. Then for any input $x$, weights $A$ and $\hat{A}$, if for any layer $i$, $\norm{A^i -\hat{A}^i}\leq \frac 1d \norm{A^i}$, then:
\begin{equation}
\norm{f_A(\x)-f_{\hat{A}}(\x)} \leq e \norm{\x}\left(\prod_{i=1}^d \norm{A^i}_2\right)\sum_{i=1}^d \frac{\norm{A^i - \hat{A}^i}_2}{\norm{A^i}_2}
\end{equation}
\label{perturbation}
\end{lemma}

\begin{proof}
Proof is contained in \cite{neyshabur2017pac}.
\end{proof}

We now have what we need to prove Theorem \ref{generalization}. For the final part of the proof, we need to compute the covering number so that we can use the Dudley entropy integral to bound the Rademacher complexity, which bounds the \emph{generalization gap}.  

\begin{proof} (Modified from \cite{arora2018stronger})
In order to compute a covering number, we need to compute the required accuracy in each parameter in the compressed network to cover the original network. Let $\W^i$ be the weights of the original net work, and $\W(t)^i$ be the weights of the compressed network. We assume that the network is balanced:

\begin{equation}
\forall i, j \quad \norm{\W^i}_F= \norm{\W^j}_F = \beta
\end{equation}

For any $x \in S$ we have:

\begin{equation}
 \beta^d = \prod_{i=1}^d \norm{\W^i}_F  \leq  \frac{c \norm{\x^1}}{\norm{\x} \mu_1}\prod_{i=2}^d \leq \frac{c^2 \norm{\x^2}}{\norm{\x}\mu_1 \mu_2}\prod_{i=3}^d \leq \frac{c^d\norm{f_A(\x)}}{\norm{\x}\prod_{i=1}^d \mu_i}
\end{equation}

Now, by Lemma \ref{induction}, we know that $\norm{\hatx{j}{i} - \x^j} \leq (i/d)\epsilon \norm{\x^j}$. $\norm{\W(t)^i}_F = \norm{\W^i - \W^i + \W(t)^i} = \norm{\W^i + \Delta^i} \leq \beta + \norm{\Delta^i} \leq \beta(1 + \frac 1d)$. Let $\hat{\W}^i$ correspond to the weights after approximating each parameter in $\W(t)^i$ with accuracy $\nu$. We can then say:

\begin{equation}
\norm{\hat{\W}^i - \W(t)^i}_F \leq \sqrt{k_i} \nu \leq \sqrt{q} \nu
\end{equation}

where $q$ is the total number of parameters. We use Lemma $\ref{perturbation}$ to bound:

\begin{gather}
|l_{\gamma}(f_{\hat{\W}}(\x), y) - l_{\gamma}(f_{\W(t)}(\x),y) | \leq e\norm{\x} \prod_{i=1}^d \norm{\W(t)^i}  \sum_{i=1}^d \frac{\norm{\W(t)^i - \hat{\W}^i}_F}{\norm{\W(t)}_F}\\
 \leq  \left( e^2 \norm{\x}\beta^{d} \beta^{-1} \right) \sum_{i=1}^d \norm{\W(t)^i - \hat{\W}^i}_F 
 \leq \left( e^2 \frac{c^d \norm{f_A(\x)}}{\prod_{i=1}^d \mu_i} \beta^{-1} \right) \sum_{i=1}^d \norm{\W(t)^i - \hat{\W}^i}_F \\
 \leq \left( e^2\frac{c^d \norm{f_A(\x)}d}{\prod_{i=1}^d \mu_i} \right) \frac{\sqrt{q}\nu}{ \beta}  
 = \frac{\sqrt{q} \nu  \kappa}{\beta} \leq \frac{q \nu \kappa}{\beta}
\end{gather}
where $\kappa$ is a placeholder for the variables in parenthesis in the final line. 

We know the maximum a parameter can be is $2\beta$ since $\norm{\tilde{\W(t)}}_F \leq \beta (1 + \frac 1d)$. Thus the log number of choices to get an $\epsilon$ cover is $\log \left(\frac {2q \kappa}{\epsilon}\right)$. Thus the covering number is $q \log \left( \frac{2 q \kappa}{\epsilon}\right)$. We know that the Rademacher complexity bounds the generalization error, and that the Dudley entropy integral bounds the Rademacher complexity.

\begin{gather}
\int_0^D \sqrt{q \log\left(\frac{2 q \kappa}{\epsilon}\right)} d\epsilon = \sqrt{q \log \left( \frac{2 q \kappa}{\epsilon}\right)}\frac 12 \left(2\epsilon - \frac{\sqrt{\pi} 2q\kappa \text{erf}\left(\sqrt{\log(2q\kappa/\epsilon)}\right)}{\sqrt{\log(2q\kappa / \epsilon)}} \right)\Bigg|^D_0
\end{gather}
Which completes the proof. 

\end{proof}

\section{Experimental Details: Plot of the Stable Rank}
\label{appendixD}

In Figure \ref{stable-rank}, the stable rank of the final dense layer in a collection of pre-trained architectures is computed as a function of $\alpha$, the exponent of a power-law fit. The spectral norm computed for the stable rank is computed through power iteration, using $1000$ iterations. The power-law fit was done using the \texttt{powerlaw} python package, which uses a maximum-likelihood fit. The data was fit to a mixture of linear regressors using expectation maximization. A total of 43 \texttt{pytorch} pre-trained architectures were used, and they are listed as follows: 

Deep Residual Learning for Image Recognition \cite{he2016deep}:\\
\texttt{resnet20\_cifar10},  
\texttt{resnet56\_cifar10}, 
\texttt{resnet110\_cifar10}, \\
\texttt{resnet164bn\_cifar10}, 
\texttt{resnet272bn\_cifar10}, 
\texttt{resnet20\_cifar100}, \\
\texttt{resnet56\_cifar100}, 
\texttt{resnet20\_cifar100},
\texttt{resnet56\_cifar100},\\
\texttt{resnet100\_cifar100},
\texttt{resnet20\_svhn},
\texttt{resnet56\_svhn},\\
\texttt{resnet110\_svhn}.

Densely Connected Convolutional Networks \cite{huang2017densely}:\\
\texttt{densenet40\_k12\_cifar10},
\texttt{densenet40\_k12\_cifar100},
\texttt{densenet40\_k12\_bc\_svhn},\\
\texttt{xdensenet40\_2\_k24\_bc\_cifar10},
\texttt{xdensenet40\_2\_k24\_bc\_svhn}.

Deep Pyramidal Residual Networks \cite{han2017deep}:\\
\texttt{pyramidnet110\_a48\_svhn},
\texttt{pyramidnet110\_a48\_cifar10},
\texttt{pyramidnet110\_a84\_cifar10}.

Identity Mappings in Deep Residual Networks \cite{he2016identity}:\\
\texttt{preresnet20\_svhn},
\texttt{preresnet56\_svhn},
\texttt{preresnet110\_svhn},\\
\texttt{preresnet20\_cifar10},
\texttt{preresnet56\_cifar10},
\texttt{preresnet110\_cifar10}.

Squeeze and Excitation Networks \cite{hu2018squeeze}:\\
\texttt{seresnet20\_cifar10},
\texttt{seresnet56\_cifar10},
\texttt{seresnet110\_cifar10},\\
\texttt{seresnet20\_svhn},
\texttt{seresnet56\_svhn},
\texttt{seresnet110\_svhn},\\
\texttt{sepreresnet20\_cifar10},
\texttt{sepreresnet56\_cifar10},
\texttt{sepreresnet110\_cifar10},\\
\texttt{sepreresnet20\_svhn},
\texttt{sepreresnet56\_cifar10},
\texttt{sepreresnet110\_cifar10}.

Wide Residual Networks \cite{zagoruyko2016wide}:\\
\texttt{wrn16\_10\_cifar10},
\texttt{wrn28\_10\_cifar10},
\texttt{wrn28\_10\_svhn},\\
\texttt{wrn40\_8\_cifar10},
\texttt{wrn40\_8\_svhn}.

\section{Experimental Details: Table of Compression Accuracy}
\label{appendixE}

In Table \ref{accuracy}, the accuracy of a model is computed before and after compressing the final dense layer with a Gaussian approximation. The mean and standard deviation are computed of the matrix elements of the weight matrix. Matrix elements less than or equal to the computed standard deviation are replaced with a Gaussian random variable with the computed mean and standard deviation. Matrix elements greater than the standard deviation are retained. This procedure is performed ten times, and the mean and standard deviation of the post-compression accuracy is computed from the $10$ trials. 

\end{document}